\documentclass[envcountsame]{llncs}

\usepackage{times}
\usepackage{latexsym,amssymb,amsmath}
\usepackage{xspace}
\usepackage{graphicx}
\usepackage{misc}
\usepackage{mathabx}
\usepackage{color} 
\usepackage{named} 
\usepackage{mathrsfs}
\usepackage{tikz}
\usetikzlibrary{graphs}

\newcommand{\greif}[1]{\ensuremath{\bigocirc#1}}
\newcommand{\lreif}[1]{\ensuremath{\bigodot#1}}

\newcommand\restr[2]{{
  \left.\kern-\nulldelimiterspace 
  #1 
  \vphantom{\big|} 
  \right|_{#2} 
  }}
  
\newcommand{\Int}[1]{#1^{\Imc}\xspace}
\renewcommand\dom{\ensuremath{\Delta}\xspace}
\newcommand{\KB}{\ensuremath{\mathcal{KB}}\xspace}
\newcommand{\per}{\mathpunct{\mbox{\bf .}}}
\newcommand{\pth}[2]{\ensuremath{\textsc{path}_{\mathscr{T}}(#1,#2)}\xspace}
\newcommand{\chd}[2]{\ensuremath{\textsc{child}_{\mathscr{T}}(#1,#2)}\xspace}

\title{Extending \DLR with Labelled Tuples, Projections, Functional
  Dependencies and Objectification\\ (full version)}

\author{Alessandro Artale \and Enrico Franconi}
\institute{KRDB Research Centre, 
Free University of Bozen-Bolzano, Italy\\
\texttt{\{artale,franconi\}@inf.unibz.it}
}

\begin{document}

\date{}
\maketitle


\begin{abstract}
We introduce an extension of the n-ary description logic \DLR to deal with attribute-labelled tuples (generalising the positional notation), with arbitrary projections of relations (inclusion dependencies), generic functional dependencies and with global and local objectification (reifying relations or their projections). We show how a simple syntactic condition on the appearance of projections and functional dependencies in a knowledge base makes the language decidable without increasing the computational complexity of the basic \DLR language.
\end{abstract}

\section{Introduction}

We introduce in this paper the language \DLRp which extends the $n$-ary description logics \DLR~\cite{calvanese:et:al:98b,BCMNP03} and \DLRID~\cite{CalvaneseGL01} as follows:

\begin{itemize}
\item the semantics is based on attribute-labelled tuples: an element of a tuple is identified by an attribute and not by its position in the tuple, e.g.,
the relation \texttt{Person} has attributes \texttt{firstname}, \texttt{lastname}, \texttt{age}, \texttt{height} with instance:\\
\texttt{$\langle$ firstname: Enrico, lastname: Franconi, age: 53, height: 1.90$\rangle$};
\item renaming of attributes is possible, e.g., to recover the positional semantics:\\ $\texttt{firstname,lastname,age,height}\looparrowright \texttt{1,2,3,4}$;
\item it can express projections of relations, and therefore inclusion dependencies, e.g.,
$\ATLEASTRS{\texttt{firstname,lastname}}\texttt{Student}\sqsubseteq\ATLEASTRS{\texttt{firstname,lastname}}\texttt{Person}$;
\item it can express multiple-attribute cardinalities, and therefore functional dependencies and multiple-attribute keys, e.g., the functional dependency from \texttt{firstname,} \texttt{lastname} to \texttt{age} in \texttt{Person} can be written as:\\
$\exists[\texttt{firstname,lastname}] \texttt{Person} \sqsubseteq$ 

\hspace{2em} $\exists^{\leq 1}[\texttt{firstname,lastname}](\exists[\texttt{firstname,lastname,age}] \texttt{Person})$;
\item it can express global and local objectification (also known as reification): a tuple may be identified  by a unique global identifier, or by an identifier which is unique only within the interpretation of a relation, e.g., to identify the name of a person we can write 
$\texttt{Name}\sqsubseteq\lreif{\exists[\texttt{firstname,lastname}] \texttt{Person}} $.
\end{itemize}

We show how a simple syntactic condition on the appearance of projections in the knowledge base makes the language decidable without increasing the computational complexity of the basic \DLR language. We call \DLRpm this fragment of \DLRp. \DLRpm is able to correctly express the UML fragment as introduced in~\cite{BeCD05-AIJ-2005,ACKRZ:er07} and the ORM fragment as introduced in~\cite{DBLP:conf/otm/FranconiM13}.


\section{Syntax of the Description Logic \DLRp}
\label{sec:syntax}

\begin{figure*}
	[t] 
	\begin{center}
		\renewcommand{\arraystretch}{1.2} $
		\begin{array}{r@{\hspace{2ex}}c@{\hspace{2ex}}l} 
			C & \to & \top\ \mid\ \bot\ \mid\ C\!N\ \mid\ \neg C\ \mid\ C_{1}\sqcap C_{2}\ \mid\ C_{1}\sqcup C_{2}\ \mid\ \EXISTR{q}{U_i} R\ \mid\ \greif{R}\ \mid\ \lreif{R\!N}\\
			R & \to & R\!N\ \mid\ R_1\setminus R_2\ \mid\ R_{1}\sqcap R_{2}\mid\ R_{1}\sqcup R_{2}\mid\ \selects{U_i}{C}{R}\ \mid\ \EXISTR{q}{U_1,\ldots,U_k} R\\
			\varphi & \to & C_1\sqsubseteq C_2\ \mid\ R_1\sqsubseteq R_2 \\
			\vartheta & \to & U_1 \looparrowright U_2 
		\end{array}
		$ 
		\renewcommand{\arraystretch}{1} 
	\end{center}
	\caption{\label{fig:dlrp} Syntax of \DLRp.} 
\end{figure*}

\begin{figure*}
	[t] 
	\begin{center}
		\renewcommand{\arraystretch}{1.2} $ { 
		\begin{array}{r@{\hspace{1ex}}l@{\hspace{3ex}}l@{\hspace{.3ex}}} 
			\tau(R_1\setminus R_2) = & \tau(R_1) & \text{if } \tau(R_1)=\tau(R_2)\\
			\tau(R_{1}\sqcap R_{2}) = & \tau(R_1) & \text{if } \tau(R_1)=\tau(R_2)\\
			\tau(R_{1}\sqcup R_{2}) = & \tau(R_1) & \text{if } \tau(R_1)=\tau(R_2)\\
			\tau(\selects{U_i}{C}{R}) = & \tau(R) & \text{if } U_i\in\tau(R)\\
			\tau(\EXISTR{q}{U_1,\ldots,U_k} R) = & \{U_1,\ldots,U_k\} & \text{if } \{U_1,\ldots,U_k\}\subset \tau(R)\\
			\tau(R) = & \emptyset & \text{otherwise} 
		\end{array}
		}$ 
		\renewcommand{\arraystretch}{1} 
	\end{center}
	\caption{\label{fig:syn:tau} The signature of \DLRp relations.} 
\end{figure*}

We first define the syntax of the language \DLRp. A \emph{signature}
in \DLRp is a triple
$\mathcal{L}=(\mathcal{C},\mathcal{R},\mathcal{U},\tau)$ consisting of
a finite set $\mathcal{C}$ of {\em concept} names (denoted by $C\!N$),
a finite set $\mathcal{R}$ of {\em relation} names (denoted by $R\!N$)
disjoint from $\mathcal{C}$, and a finite set $\mathcal{U}$ of {\em
  attributes} (denoted by $U$), and a \emph{relation signature}
function $\tau$ associating a set of attributes to each relation name,
$\tau(R\!N)=\{U_1,\ldots,U_n\}\subseteq \Umc$ with $n\geq 2$.

The syntax of concepts $C$, relations $R$, formulas $\varphi$, and
attribute renaming axioms $\vartheta$ is defined in
Figure~\ref{fig:dlrp}, where $q$ is a positive integer and $2\leq k < \textsc{arity}(R)$. We extend the signature function $\tau$ to arbitrary relations as specified in Figure~\ref{fig:syn:tau}. We define the \textsc{arity} of a relation $R$ as the number of the attributes in its signature, namely $\left|\tau(R)\right|$.

A \DLRp \emph{TBox} \Tmc is a finite set of formulas, i.e., \emph{concept inclusion} axioms of the form $C_1\sqsubseteq C_2$ and \emph{relation inclusion} axioms of the form $R_1\sqsubseteq R_2$.
\\
A renaming schema induces an equivalence relation $(\looparrowright,\mathcal{U})$ over the attributes $\mathcal{U}$, providing a partition of $\mathcal{U}$ into equivalence classes each one representing the alternative ways to name attributes. We write $[U]_\Re$ to denote the equivalence class of the attribute $U$ w.r.t. the equivalence relation $(\looparrowright,\mathcal{U})$. 
We allow only \emph{well founded} renaming schemas, namely schemas such that each equivalence class $[U]_\Re$ in the induced equivalence relation never contains two attributes from the same relation signature. 
In the following we use the shortcut $U_1\ldots U_n\looparrowright U'_1\ldots U'_n$ to group many renaming axioms, with the obvious meaning that $U_i\looparrowright U'_i$, for all $i=1,\ldots, n$.

A \DLRp knowledge base $\mathcal{KB}=(\mathcal{T},\Re)$ is composed by a TBox $\mathcal{T}$ and a renaming schema $\Re$.


The renaming schema reconciles the attribute and the positional perspectives on relations (see also the similar perspectives in relational databases~\cite{AbiteboulHV95}). They are crucial when expressing both inclusion axioms and operators ($\sqcap,~\sqcup,~\setminus$) between relations, which make sense only over \emph{union compatible} relations. Two relations $R_1,R_2$ are union compatible if their signatures are equal up to the attribute renaming induced by the renaming schema $\Re$, namely, $\tau(R_1)=\{U_1,\ldots,U_n\}$ and $\tau(R_2)=\{V_1,\ldots,V_n\}$ have the same arity $n$ and $[U_i]_\Re=[V_i]_\Re$ for each $1\leq i\leq n$. Notice that, thanks to the renaming schema, relations can use just local attribute names that can then be renamed when composing relations.
Also note that it is obviously possible for the same attribute to appear in the signature of different relations.

\vspace{2ex}

To show the expressive power of the language, let us consider the following example with tree relation names $R_1, R_2$ and $R_3$ with the following signature:
\begin{align*}
  \tau(R_1)  &= \{U_1,U_2,U_3,U_4,U_5\}\\
  \tau(R_2)  &= \{V_1,V_2,V_3,V_4,V_5\}\\
  \tau(R_3)  &= \{W_1,W_2,W_3,W_4\}
\end{align*}
To state that $\{U_1,U_2\}$ is the \emph{multi-attribute key} of $R_1$ we add the axiom:
  \begin{align*}
    \exists[U_1,U_2] R_1 \sqsubseteq \exists^{\leq 1}[U_1,U_2] R_1
  \end{align*}
where $\exists[U_1,\ldots,U_k] R$ stands for $\exists^{\geq 1}[U_1,\ldots,U_k] R$. To express that there is a \emph{functional dependency} from the attributes $\{V_3,V_4\}$ to the attribute $\{V_5\}$ of $R_2$ we add the axiom:
\begin{align}\label{funct-dep}
      \exists[V_3,V_4] R_2 \sqsubseteq \exists^{\leq 1}[V_3,V_4](\exists[V_3,V_4,V_5] R_2)
\end{align}
The following axioms express that $R_2$ is a sub-relation of $R_1$ and
that a projection of $R_3$ is a sub-relation of a projection of $R_1$,
together with the corresponding axioms for the
renaming schema to explicitly specify the 
correspondences between the attributes of the two inclusion dependencies:
\begin{align*}
  R_2 &\sqsubseteq R_1\\
  \exists[W_1,W_2,W_3] R_3 &\sqsubseteq \exists[U_3,U_4,U_5] R_1\\
  V_1V_2V_3V_4V_5 &\looparrowright U_1U_2U_3U_4U_5 \\
  W_1W_2W_3 &\looparrowright U_3U_4U_5
\end{align*}


\section{Semantics}  

\begin{figure*}
	[t] 
	\begin{center}
		\renewcommand{\arraystretch}{1.2} $ { 
		\begin{array}{r@{\hspace{1ex}}l@{\hspace{.3ex}}} 
			\Int{\top} = & \dom\\
			\Int{\bot} = & \emptyset\\
			\Int{(\neg C)} = & \Int{\top} \setminus \Int{C}\\
			\Int{(C_{1}\sqcap C_{2})} = & \Int C_{1} \cap \Int C_{2}\\
			\Int{(C_{1}\sqcup C_{2})} = & \Int C_{1} \cup \Int C_{2}\\
			\Int{(\EXISTR{q}{U_i} R)} = & \{d\in\dom\mid~\left|\{t\in\Int R\mid t[\rho(U_i)]=d\}\right| \lesseqgtr q \}\\
			\Int{(\greif{R})} = & \{d\in \dom \mid d=\imath(t) \land t\in \Int{R}\}\\
			\vspace{2ex}
			\Int{(\lreif{R\!N})} = & \{d\in \dom \mid d=\ell_{R\!N}(t)\land t\in \Int{R\!N}\}\\
			\Int{(R_1\setminus R_2)} = & \Int R_{1} \setminus \Int R_{2}\\
			\Int{(R_{1}\sqcap R_{2})} = & \Int R_{1} \cap \Int R_{2}\\
			\Int{(R_{1}\sqcup R_{2})} = & \{t\in\Int R_{1}\cup\Int R_{2}\mid \rho(\tau(R_1))= \rho(\tau(R_2))\}\\
			\Int{(\selects {U_i}{C}{R})} = & \{\parbox[t]{
			\textwidth}{$ t\in\Int{R} \mid t[\rho(U_i)]\in\Int{C}\} $} \\
			\Int{(\EXISTR{q}{U_1,\ldots,U_k} R)} = & \{ \parbox[t]{
			\textwidth}{$ \langle \rho(U_1):d_1,\ldots,\rho(U_k):d_k\rangle \in T_{\dom}(\{\rho(U_1),\ldots,\rho(U_k)\}) \mid~\vspace{0.5ex}\\
			\left|\{t\in\Int R \mid t[\rho(U_1)]=d_1,\ldots,t[\rho(U_k)]=d_k\}\right| \lesseqgtr q \} $} 
		\end{array}
		}$ 
		\renewcommand{\arraystretch}{1} 
	\end{center}
	\caption{\label{fig:sem:dlrp} Semantics of \DLRp expressions.} 
\end{figure*}

The semantics makes use of the notion of \emph{labelled tuples} over a domain set $\Delta$: a \emph{$\mathcal{U}$-labelled tuple over $\Delta$} is a function $t \colon \mathcal{U} \to \Delta$. For $U\in \mathcal{U}$, we write $t[U]$ to refer to the domain element ${d\in \Delta}$ labelled by $U$, if the function $t$ is defined for $U$ -- that is, if the attribute $U$ is a label of the tuple $t$. Given $d_1,\dots,d_n\in \Delta$, the expression ${\langle U_1\colon d_1,\ldots,U_n\colon d_n\rangle}$ stands for the $\mathcal{U}$-labelled tuple $t$ over $\Delta$ (tuple, for short) such that ${t[U_i]=d_i}$, for ${1\leq 1\leq n}$. 
We write ${t[U_1,\ldots,U_k]}$ to denote the \emph{projection} of the tuple $t$ over the attributes ${U_1,\ldots,U_k}$, namely the function $t$ restricted to be undefined for the labels not in ${U_1,\ldots,U_k}$. The set of all $\mathcal{U}$-labelled tuples over $\Delta$ is denoted by $T_\Delta(\mathcal{U})$.

A \DLRp \emph{interpretation}, $\Imc = (\dom, \cdot^\Imc, \rho, \imath, \ell_{R\!N_1}, \ell_{R\!N_2},\ldots)$, consists of a nonempty \emph{domain} $\dom$, an \emph{interpretation function} $\cdot^\Imc$, a \emph{renaming function} $\rho$, a \emph{global objectification function} $\imath$, and a family of \emph{local objectification functions} $\ell_{R\!N_i}$, one for each named relation $R\!N_i\in\mathcal{R}$.

The renaming function $\rho$ for attributes is a total function ${\rho:\mathcal{U}\to\mathcal{U}}$ representing a canonical renaming for all attributes. We consider, as a shortcut, the notation $\rho(\{U_1,\ldots,U_k\}) = \{\rho(U_1),\ldots,\rho(U_k)\}$.
\\
The global objectification function is an injective function, ${\imath:T_{\dom}(\Umc) \to \dom}$, associating a \emph{unique} global identifier to each possible tuple.
%
\\
The local objectification functions, ${\ell_{R\!N_i}:T_{\dom}(\Umc) \to \dom}$, are distinct for each relation name in the signature, and as the global objectification function they are injective: they associate an identifier -- which is unique only within the interpretation of a relation name -- to each possible tuple.
\\
The interpretation function $\cdot^\Imc$ assigns a set of domain elements to each concept name, $C\!N^{\mathcal{I}}\subseteq \dom$, and a set of $\Umc$-labelled tuples over $\dom$ to each relation name conforming with its signature and the renaming function: $$R\!N^{\mathcal{I}}\subseteq T_{\dom}(\{\rho(U)\mid U\in\tau(R\!N)\}).$$ 

The interpretation function $\cdot^\Imc$ is unambiguously extended over concept and relation expressions as specified in the inductive definition of Fig.~\ref{fig:sem:dlrp}.

An interpretation $\Imc$ satisfies a concept inclusion axiom $C_1\sqsubseteq C_2$ if $\Int C_1\subseteq \Int C_2$, it satisfies a relation inclusion axiom $R_1\sqsubseteq R_2$ if $\Int R_1\subseteq \Int R_2$, and it satisfies a renaming schema $\Re$ if the renaming function $\rho$ renames the attributes in a consistent way with respect to $\Re$, namely if
$$\forall U\per\rho(U)\in[U]_\Re\land\forall V\in [U]_\Re\per\rho(U)=\rho(V).$$

An interpretation is a \emph{model} for a knowledge base $(\mathcal{T},\Re)$ if it satisfies all the formulas in the TBox $\mathcal{T}$ and it satisfies the renaming schema $\Re$. We define \emph{KB satisfiability} as the problem of deciding the existence of a model of a given knowledge base, \emph{concept satisfiability} (resp. \emph{relation satisfiability}) as the problem of deciding whether there is a model of the knowledge base that assigns a non-empty extension to a given concept (resp. relation), and \emph{entailment} as the problem to check whether a given knowledge base logically implies a formula, that is, whenever all the models of the knowledge base are also models of the formula.
\\
For example, from the knowledge base $\mathcal{KB}$ introduced in the previous Section the following logical implication holds:
\begin{align*}
  \mathcal{KB}\models \exists[V_1,V_2] R_2 \sqsubseteq \exists^{\leq 1}[V_1,V_2] R_2
\end{align*}
i.e., the attributes $V_1,V_2$ are a key for the relation $R_2$.

\begin{proposition}
	The problems of {KB satisfiability}, {concept and relation satisfiability}, and {entailment} are mutually reducible in \DLRp. 
\end{proposition}

\begin{proof}
	We first show that we can reduce all the problems to concept satisfiability, where a concept $C$ is satisfiable iff $\KB \nvDash C\sqsubseteq \bot$. 
	\begin{itemize}
		\item \KB is satisfiable iff $\KB \nvDash \top\sqsubseteq \bot$; 
		\item $\KB \models C_1 \sqsubseteq C_2$ iff $\KB \models C_1 \sqcap \neg C_2\sqsubseteq \bot$; 
		\item $\KB \models R_1 \sqsubseteq R_2$ iff $\KB \models \exists[U](R_1\sqcap \neg R_2)\sqsubseteq \bot$, for some $U\in\tau(R_1)$; 
		\item $\KB \nvDash R \sqsubseteq \bot$ iff $\KB \nvDash \exists[U]R\sqsubseteq \bot$, for some $U\in\tau(R)$. 
	\end{itemize}
	Viceversa, we can show that concept satisfiability can be reduced to any other problem. First, note that concept satisfiability is already expressed as a logical implication problem. For the other cases, given a fresh new binary relation $P$, we have that 
	\begin{itemize}
		\item $\KB \nvDash C\sqsubseteq \bot$ iff $\KB \cup \{\top \sqsubseteq \exists[U_1](P\sqcap \sigma_{U_2:C}P)\}$ is satisfiable; 
		\item $\KB \nvDash C\sqsubseteq \bot$ iff $\KB\nvDash \sigma_{U_2:C}P \sqsubseteq \bot$.\hfill\qed 
	\end{itemize}
\end{proof}

\DLRp can express complex inclusion and functional dependencies, for which it is well known that reasoning is undecidable~\cite{Mitchell83,ChandraV85}. \DLRp also includes the \DLR extension \DLRID together with unary functional dependencies~\cite{CalvaneseGL01}, which also has been proved to be undecidable.


\section{The \DLRpm fragment of \DLRp}

Given a \DLRp knowledge base $(\Tmc,\Re)$, we define the \emph{projection signature} as the set $\mathscr{T}$ including the signatures $\tau(R\!N)$ of the relations $R\!N\in\mathcal{R}$, the singletons associated with each attribute name $U\in\mathcal{U}$, and the relation signatures as they appear explicitly in projection constructs in the relation inclusion axioms of the knowledge base, together with their implicit occurrences due to the renaming schema:
  
\begin{enumerate}
	\item $\tau(R\!N)\in\mathscr{T}$ \text{ if } $R\!N\in\mathcal{R}$; 
	\item $\{U\}\in\mathscr{T}$ \text{ if } $U\in\mathcal{U}$; 
	\item $\{U_1,\ldots,U_k\}\in\mathscr{T}$ \text{ if } $\EXISTR{q}{V_1,\ldots,V_k} R\in\Tmc$ \text{ and } $\{U_i,V_i\}\subseteq [U_i]_\Re $ \text{ for}~$1\!\leq\!i\!\leq\!k$. 
\end{enumerate}

\begin{figure*}[t]
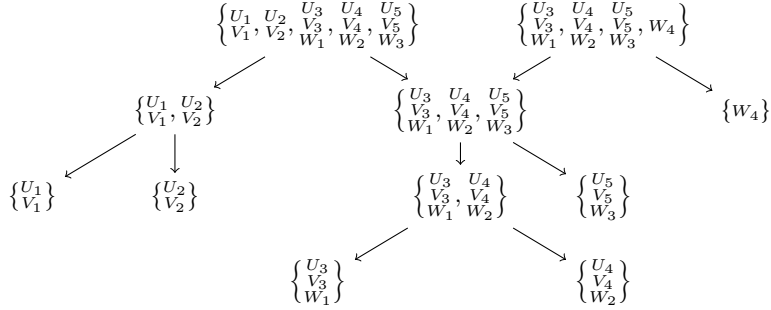

\begin{center}
\tikz[x=6em,y=8ex] {
\node (a) at (2,3)  {$\left\{\substack{U_1\\V_1},\substack{U_2\\V_2},\substack{U_3\\V_3\\W_1},\substack{U_4\\V_4\\W_2},\substack{U_5\\V_5\\W_3}\right\}$};
\node (b) at (4,3)  {$\left\{\substack{U_3\\V_3\\W_1},\substack{U_4\\V_4\\W_2},\substack{U_5\\V_5\\W_3},\substack{\vspace{1ex}\\W_4\\~}\right\}$};
\node (c) at (1,2)  {$\left\{\substack{U_1\\V_1},\substack{U_2\\V_2}\right\}$};
\node (c1) at (4,1)  {$\left\{\substack{U_5\\V_5\\W_3}\right\}$};
\node (d) at (3,2)  {$\left\{\substack{U_3\\V_3\\W_1},\substack{U_4\\V_4\\W_2},\substack{U_5\\V_5\\W_3}\right\}$};
\node (d1) at (5,2)  {$\left\{\substack{~\\W_4}\right\}$};
\node (e) at (0,1)  {$\left\{\substack{U_1\\V_1}\right\}$};
\node (f) at (1,1)  {$\left\{\substack{U_2\\V_2}\right\}$};
\node (g) at (3,1)  {$\left\{\substack{U_3\\V_3\\W_1},\substack{U_4\\V_4\\W_2}\right\}$};
\node (h) at (2,0)  {$\left\{\substack{U_3\\V_3\\W_1}\right\}$};
\node (i) at (4,0)  {$\left\{\substack{U_4\\V_4\\W_2}\right\}$};
\draw 
(a) edge[->] (c)
(d) edge[->] (c1)
(a) edge[->] (d)
(b) edge[->] (d)
(b) edge[->] (d1)
(c) edge[->] (e)
(c) edge[->] (f) 
(d) edge[->] (g)
(g) edge[->] (h) 
(g) edge[->] (i);
}
\end{center}
\caption{\label{fig:multitree} The projection signature graph of the example.}
\end{figure*}

We call \emph{projection signature graph} the directed acyclic graph $(\supset,\mathscr{T})$ with the attribute singletons $\{U\}$ being the sinks. The \DLRpm fragment of \DLRp allows only for knowledge bases with a projection signature graph being a \emph{multitree}, namely the set of nodes reachable from any node of the projection signature graph should form a tree.  
Given a relation name ${R\!N}$, the subgraph of the projection signature graph dominated by ${R\!N}$ is a tree where the leaves are all the attributes in $\tau({R\!N})$ and the root is $\tau({R\!N})$.
\\
We call $\mathscr{T}_{\{U_1,\ldots,U_k\}}$ the tree formed by the nodes in the projection signature graph dominated by the set of attributes $\{U_1,\ldots,U_k\}$.
Given two relation signatures (i.e., two sets of attributes) $\tau_1,\tau_2\subseteq\mathcal{U}$, by $\pth{\tau_1}{\tau_2}$ we denote the path in $(\supset,\mathscr{T})$ between $\tau_1$ and $\tau_2$, if it exists.  Note that $\pth{\tau_1}{\tau_2}=\emptyset$ both when a path does not exist and when $\tau_1\subseteq \tau_2$, and
$\textsc{path}_{\mathscr{T}}$ is functional in \DLRpm due to the multitree restriction on projection signatures. The notation $\chd{\tau_1}{\tau_2}$ means that ${\tau_2}$ is a child of ${\tau_1}$ in $(\supset,\mathscr{T})$.

In addition to the above multitree condition, the \DLRpm fragment of
\DLRp allows for knowledge bases with projection constructs
$\EXISTR{q}{U_1,\ldots,U_k} R$ (resp. $\EXISTR{q}{U} R$) with a
cardinality $q>1$ only if the length of the path
$\pth{\{U_1,\ldots,U_k\}}{\tau(R)}$ (resp. $\pth{\{U\}}{\tau(R)}$) is
1. This allows to map cardinalities in \DLRpm
into cardinalities in \ALCQI.

Figure~\ref{fig:multitree} shows that the projection signature graph
of the knowledge base introduced in Section~\ref{sec:syntax} is indeed
a multitree. Note that in the figure we have collapsed equivalent
attributes in a unique equivalence class, according to the renaming
schema.

\DLRpm restricts \DLRp only in the way multiple projections of relations appear in the knowledge base. It is easy to see that \DLR is included in \DLRpm, since the projection signature graph of any \DLR knowledge base has maximum depth equal to 1. \DLRID~\cite{CalvaneseGL01} together with (unary) functional dependencies is also included in \DLRpm, with the proviso that projections of relations in the knowledge base form a multitree projection signature graph. Since (unary) functional dependencies are expressed via the inclusions of projections of relations (see, e.g., the functional dependency~(\ref{funct-dep}) in the previous example), by constraining the projection signature graph to be a multitree, the possibility to build combinations of functional dependencies as the ones in~\cite{CalvaneseGL01} leading to undecidability is ruled out. Also note that \DLRpm is able to correctly express the UML fragment as introduced in~\cite{BeCD05-AIJ-2005,ACKRZ:er07} and the ORM fragment as introduced in~\cite{DBLP:conf/otm/FranconiM13}.


\section{Mapping \DLRpm to \ALCQI}
\label{sec:mapping}

We show that reasoning in \DLRpm is \ExpTime-complete by providing a mapping from \DLRpm knowledge bases to \ALCQI knowledge bases; the reverse mapping from \ALCQI knowledge bases to \DLR knowledge bases is well known. The proof is based on the fact that reasoning with \ALCQI knowledge bases is \ExpTime-complete~\cite{BCMNP03}. We adapt and extend the mapping presented for \DLR in~\cite{calvanese:et:al:98b}.


In the following we use the shortcut $(S_1\chain\ldots\chain S_n)^-$ for $S_n^-\chain\ldots\chain S_1^-$, the shortcut $\exists^{\lesseqgtr 1} S_1\chain\ldots\chain S_n\per C$ for $\card{1}{S_1\per\ldots\per\exists^{\lesseqgtr 1} S_n}{C}$ and the shortcut $\forall S_1\chain\ldots\chain S_n\per C$ for $\forall S_1\per\ldots\per\forall S_n\per {C}$. 
Note that these shortcuts for the role chain constructor ``$\chain$'' are not correct in general, but they are correct in the context of the specific \ALCQI knowledge bases used in this paper.

Let $\KB = (\mathcal{T},\Re)$ be a \DLRpm knowledge base. 
We first rewrite the knowledge base as follows: for each equivalence class $[U]_{\Re}$ a single \emph{canonical} representative of the class is chosen, and the \KB is consistently rewritten by substituting each attribute with its canonical representative. After this rewriting, the renaming schema does not play any role in the mapping.


\begin{figure*}[t]
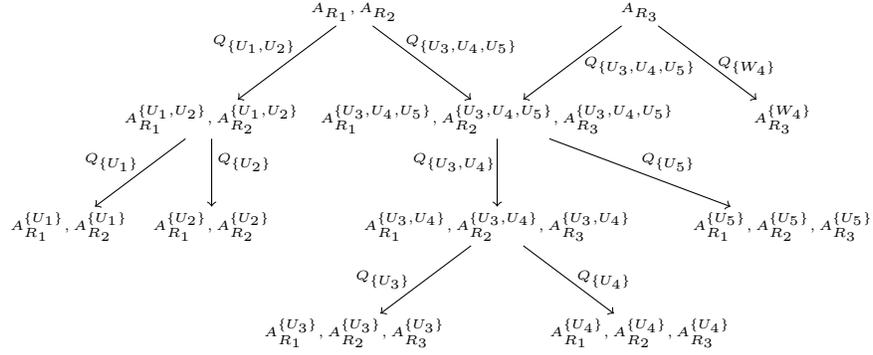

\begin{center}
\tikz[x=6em,y=10ex] {\tiny
\node (a) at (2,3)  {$A_{R_1}, A_{R_2}$};
\node (b) at (4,3)  {$A_{R_3}$};
\node (c) at (1,2)  {$A_{R_1}^{\{U_1,U_2\}}, A_{R_2}^{\{U_1,U_2\}}$};
\node (c1) at (5,1)  {$A_{R_1}^{\{U_5\}}, A_{R_2}^{\{U_5\}}, A_{R_3}^{\{U_5\}}$};
\node (d) at (3,2)  {$A_{R_1}^{\{U_3,U_4,U_5\}}, A_{R_2}^{\{U_3,U_4,U_5\}}, A_{R_3}^{\{U_3,U_4,U_5\}}$};
\node (d1) at (5,2)  {$A_{R_3}^{\{W_4\}}$};
\node (e) at (0,1)  {$A_{R_1}^{\{U_1\}}, A_{R_2}^{\{U_1\}}$};
\node (f) at (1,1)  {$A_{R_1}^{\{U_2\}}, A_{R_2}^{\{U_2\}}$};
\node (g) at (3,1)  {$A_{R_1}^{\{U_3,U_4\}}, A_{R_2}^{\{U_3,U_4\}}, A_{R_3}^{\{U_3,U_4\}}$};
\node (h) at (2,0)  {$A_{R_1}^{\{U_3\}}, A_{R_2}^{\{U_3\}}, A_{R_3}^{\{U_3\}}$};
\node (i) at (4,0)  {$A_{R_1}^{\{U_4\}}, A_{R_2}^{\{U_4\}}, A_{R_3}^{\{U_4\}}$};
\node (ac) at (1.3,2.7) {$Q_{{\{U_1,U_2\}}}$};
\node (ad) at (2.75,2.7) {$Q_{\{U_3,U_4,U_5\}}$};
\node (bd) at (4,2.5) {$Q_{\{U_3,U_4,U_5\}}$};
\node (bd1) at (4.75,2.5) {$Q_{\{W_4\}}$};
\node (ce) at (0.3,1.6) {$Q_{{\{U_1\}}}$};
\node (cf) at (1.23,1.6) {$Q_{{\{U_2\}}}$};
\node (dg) at (2.7,1.6) {$Q_{{\{U_3,U_4\}}}$};
\node (dc1) at (4.2,1.6) {$Q_{{\{U_5\}}}$};
\node (gh) at (2.2,0.5) {$Q_{\{U_3\}}$};
\node (gi) at (3.75,0.5) {$Q_{\{U_4\}}$};
\draw 
(a) edge[->] (c)
(a) edge[->] (d)
(d) edge[->] (c1)
(b) edge[->] (d)
(b) edge[->] (d1) 
(c) edge[->] (e)
(c) edge[->] (f) 
(d) edge[->] (g)
(g) edge[->] (h) 
(g) edge[->] (i);
}
\end{center}
\caption{\label{fig:mapping} The \ALCQI signature generated by the example.}
\end{figure*} 


The mapping function $\cdot^\dag$ maps each concept name $C\!N$ in the \DLRpm knowledge base to an \ALCQI concept name $C\!N$, each relation name $R\!N$ in the \DLRpm knowledge base to an \ALCQI concept name $A_{R\!N}$ (its global reification), and each attribute name $U$ in the \DLRpm knowledge base to an \ALCQI role name, as detailed below.
\\
For each relation name $R\!N$ the mapping introduces a concept name $A_{R\!N}^{l}$ and a role name $Q_{R\!N}$ (to capture the local reification), and a concept name $A_{R\!N}^{\tau_i}$ for each projected signature $\tau_i$ in the projection signature graph dominated by $\tau(R\!N)$, $\tau_i\in\mathscr{T}_{\tau(R\!N)}$ (to capture global reifications of the projections of $R\!N$). Note that $A_{R\!N}^{\tau(R\!N)}$ coincides with $A_{R\!N}$. 
Furthermore, the mapping introduces a role name $Q_{\tau_i}$ for each projected signature $\tau_i$ in the projection signature, $\tau_i\in\mathscr{T}$, such that there exists $\tau_j\in\mathscr{T}$ with $\chd{\tau_j}{\tau_i}$, i.e., we exclude the case where $\tau_i$ is one of the roots of the multitree induced by the projection signature.\\
The mapping $\cdot^\dag$ applies also to a path. Let
$\tau,\tau'\in\mathscr{T}$ be two generic sets of attributes such that
the function $\pth{\tau}{\tau'} = \tau,\tau_1,\ldots,\tau_n, \tau'$,
then, a path is mapped as follows:
\begin{align*}
	\pth{\tau}{\tau'}^\dag = Q_{\tau_1}\chain\ldots\chain Q_{\tau_n}\chain Q_{\tau'}. 
\end{align*}


Intuitively, the mapping reifies each node in the projection signature graph: the target \ALCQI signature of the example of the previous section is partially presented in Fig.~\ref{fig:mapping}, together with the projection signature graph. Each node is labelled with the corresponding (global) reification concept ($A_{R_i}^{\tau_j}$), for each relation name $R_i$ and each projected signature $\tau_j$ in the projection signature graph dominated by $\tau(R_i)$, while the edges are labelled by the roles ($Q_{\tau_i}$) needed for the reification.

\begin{figure*}
	[t] 
	\begin{center}
		\renewcommand{\arraystretch}{1.2} $ 
		\begin{array}{r@{\hspace{1ex}}c@{\hspace{1ex}}l} 
			(\neg C)^\dag &=& \neg C^\dag \\
			(C_1 \sqcap C_2)^\dag & =& C_1^\dag \sqcap C_2^\dag \\
			(C_1 \sqcup C_2)^\dag & =& C_1^\dag \sqcup C_2^\dag \\
			(\EXISTR{q}{U_i} R)^\dag & =& \card{q}{\left(\pth{\tau(R)}{\{U_i\}}^\dag\right)^-}{R^\dag}\\
			(\greif{R})^\dag &=& R^\dag \\
			(\lreif{R\!N})^\dag &=& A^l_{R\!N}
			\vspace{2ex}\\
			(R_1\setminus R_2)^\dag &=& R_1^\dag\sqcap \neg R_2^\dag\\
			(R_1 \sqcap R_2)^\dag & =& R_1^\dag \sqcap R_2^\dag\\
			(R_1 \sqcup R_2)^\dag & =& R_1^\dag \sqcup R_2^\dag\\
			(\selects{U_i}{C}{R})^\dag & =& R^\dag \sqcap \forall \pth{\tau(R)}{\{U_i\}}^\dag \per C^\dag\\
			(\EXISTR{q}{U_1,\ldots,U_k} R)^\dag & =&
			\card{q}{\left(\pth{\tau(R)}{\{U_1,\ldots,U_k\}}^\dag\right)^-}{R^\dag} 
		\end{array}
		$ 
		\renewcommand{\arraystretch}{1} 
	\end{center}
	\caption{\label{fig:themapping} The mapping for concept and relation expressions.} 
\end{figure*}

The mapping $\cdot^\dag$ is extended to concept and relation
expressions as in Figure~\ref{fig:themapping}, with the proviso that
whenever $\pth{\tau_1}{\tau_2}$ returns an empty path then the
translation for the corresponding expression becomes the bottom
concept. Note that in \DLRpm the cardinalities on a path are
restricted to the case $q=1$ whenever a path is of length greater than
$1$, so we still remain within the \ALCQI description logic when the
mapping applies to cardinalities. So, if we need to express a cardinality constraint
$\EXISTR{q}{U_i} R$,] with $q>1$, then $U_i$ should not be mentioned in
any other projection of the relation $R$ in such a way that $|\pth{\tau(R)}{\{U_i\}}|=1$.

In order to explain the need for the path function in the mapping, notice that a relation is reified according to the decomposition dictated by projection signature graph it dominates. Thus, to access an attribute $U_j$ of a relation ${R_i}$ it is necessary to follow the path through the projections that use that attribute. This path is a role chain from the signature of the relation (the root) to the attribute as returned by the $\pth{\tau(R_i)}{U_i}$ function. For example, considering Fig.~\ref{fig:mapping}, in order to access the attribute $U_4$ of the relation $R_3$ in the expression $(\selects{U_4}{C}{R_3})$, the path $\pth{\tau(R_3)}{\{U_4\}}^\dag$ is equal to the role chain $Q_{\{U_3,U_4,U_5\}}\chain Q_{\{U_3,U_4\}}\chain Q_{\{U_4\}}$, so that
$
(\selects{U_4}{C}{R_3})^\dag ~=~ A_{R_3} \sqcap \forall Q_{\{U_3,U_4,U_5\}}\chain
Q_{\{U_3,U_4\}}\chain Q_{\{U_4\}}\per C.
$
\\
Similar considerations can be done when mapping cardinalities over relation projections.

The mapping $\gamma(\KB)$ of a \DLRpm knowledge base \KB with a signature $(\mathcal{C},\mathcal{R},\mathcal{U},\tau)$ is defined as the following \ALCQI TBox:
%
%
\begin{align*}
  \gamma(\KB) \quad =& \quad \gamma_\textit{dsj} ~\cup %
  \bigcup_{R\!N\in\Rmc}\gamma_{\textit{rel}}(R\!N) ~\cup %
  \bigcup_{R\!N\in\Rmc}\gamma_{\textit{lobj}}({R\!N})
  ~\cup \\
  & \quad\bigcup_{C_1\sqsubseteq C_2\in\KB}{C_1^\dag\sqsubseteq C_2^\dag}
  ~\cup%
 \bigcup_{R_1\sqsubseteq R_2\in\KB}{R_1^\dag\sqsubseteq R_2^\dag}%
\end{align*}
where
\begin{align*}
\gamma_\textit{dsj} = 
&~\bigl\{A_{R\!N_1}^{\tau_i}\sqsubseteq\neg A_{R\!N_2}^{\tau_j} \mid R\!N_1, R\!N_2\in\Rmc,
  \tau_i, \tau_j\in \mathscr{T}, |\tau_i|\geq 2, |\tau_j|\geq 2, \tau_i\neq \tau_j
  \bigr\}\\
%
  \gamma_{\textit{rel}}(R\!N)
  =&~\bigcup_{\tau_i\in\mathscr{T}_{\tau(R\!N)}}~
     \bigcup_{\chd{\tau_i}{\tau_j}}\bigl\{
     A^{\tau_i}_{R\!N} \sqsubseteq \exists Q_{\tau_j}\per A^{\tau_j}_{R\!N},~\exists^{\geq 2}
     Q_{\tau_j}\per\top\sqsubseteq \bot
     \bigr\} \\
%
\gamma_{\textit{lobj}}({R\!N}) =&~\{\parbox[t]{\textwidth}{$
A_{R\!N}\sqsubseteq\exists Q_{R\!N}\per A_{R\!N}^l,~
\exists^{\geq 2}Q_{R\!N}\per \top\sqsubseteq \bot,\\
A_{R\!N}^l \sqsubseteq \exists Q_{R\!N}^-\per A_{R\!N},~
\exists^{\geq 2} Q_{R\!N}^-\per \top\sqsubseteq \bot\}.$}
%
%
\end{align*}

Intuitively, $\gamma_\textit{dsj}$ ensures that relations with different signatures are disjoint, thus, e.g., enforcing the union compatibility. The axioms in $\gamma_{\textit{rel}}$ introduce classical reification axioms for each relation and its relevant projections. The axioms in $\gamma_{\textit{lobj}}$ make sure that each local objectification differs form the global one.

Clearly, the size of $\gamma(\KB)$ is polynomial in the size of $\KB$ (under the same coding of the numerical parameters), and thus we are able to state the main result of this paper.

\begin{theorem}
	\label{th:sat} A \DLRpm knowledge base $\KB$ is satisfiable iff the \ALCQI knowledge base $\gamma(\KB)$ is satisfiable. 
\end{theorem}

\begin{proof}
  We assume that the \KB is consistently rewritten by substituting
  each attribute with its canonical representative, thus, we do not
  have to deal with the renaming of attributes. Furthermore, we extend
  the function $\imath$ to
  singleton tuples with the meaning that $\imath(\langle U_i:d_i\rangle)=d_i$.\\
  ($\Rightarrow$) Let
  $\Imc = (\Delta^\Imc, \cdot^\Imc, \rho, \imath,
  \ell_{R\!N_1},\ldots)$
  be a model for a \DLRpm knowledge base $\KB$. To construct a model
  $\Jmc=(\Delta^\Jmc,\cdot^\Jmc)$ for the \ALCQI knowledge base
  $\gamma(\KB)$ we set $\Delta^\Jmc = \Delta^\Imc$.  Furthermore, we
  set: $(C\!N^\dag)^\Jmc=(C\!N)^\Imc$, for every atomic concept
  $C\!N\in\Cmc$, while for every ${R\!N}\in\Rmc$ and
  ${\tau_i\in\mathscr{T}_{\tau(R\!N)}}$ we set
  \begin{multline}\label{eq:R}
    (A_{R\!N}^{\tau_i})^\Jmc = \{\imath(\langle U_1:d_1,\ldots,U_k:d_k\rangle) \mid 
    \{U_1,\ldots,U_k\}=\tau_i \text{ and } \\ \exists t\in R\!N^\Imc\per t[U_1]=d_1,\ldots,t[U_k]=d_k\}.
  \end{multline}
  For each role name $Q_{\tau_i}$, $\tau_i\in\mathscr{T}$, we set
  \begin{multline}\label{eq:Q}
    (Q_{\tau_i})^\Jmc = \{(d_1,d_2)\in \Delta^\Jmc\times
    \Delta^\Jmc\mid \exists t\in {R\!N}^\Imc \text{ s.t. } d_1=\imath(t[{\tau_j}]), d_2=\imath(t[{\tau_i}])\\
    \text{ and } \chd{\tau_j}{\tau_i}, \text{ for some } {R\!N}\in\R\}.
  \end{multline}
  For every ${R\!N}\in\Rmc$ we set
  \begin{multline}\label{eq:loc}
  Q_{R\!N}^\Jmc = \{(d_1,d_2)\in \Delta^\Jmc\times \Delta^\Jmc\mid
    \exists t\in\Int{R\!N} \text{ s.t. } d_1=\imath(t) \text{ and } d_2=\ell_{R\!N}(t)\},
  \end{multline}
and
  \begin{align}\label{eq:Aloc}
  (A^l_{R\!N})^\Jmc = \{\ell_{R\!N}(t)\mid t\in \Int{R\!N}\}.
  \end{align}
We now show that $\Jmc$ is indeed a model of $\gamma(\KB)$.
\begin{enumerate}
\item $\Jmc\models \gamma_\textit{dsj}$. This is a direct consequence
  of the fact that $\imath$ is an injective function and that tuples
  with different aryties are different tuples.
\item $\Jmc\models \gamma_{\textit{rel}}(R\!N)$, for every
  ${R\!N\in\Rmc}$. We show that, for each $\tau_i,\tau_j$
  s.t. $\chd{\tau_i}{\tau_j}$ and $\tau_i\in\mathscr{T}_{\tau(R\!N)}$,
  $\Jmc\models A^{\tau_i}_{R\!N} \sqsubseteq \exists Q_{\tau_j}\per
  A^{\tau_j}_{R\!N}$
  and
  $\Jmc\models~\exists^{\geq 2} Q_{\tau_j}\per\top\sqsubseteq \bot$:
  \begin{itemize}
  \item
    $\Jmc\models A^{\tau_i}_{R\!N} \sqsubseteq \exists Q_{\tau_j}\per
    A^{\tau_j}_{R\!N}$.
    Let $d\in(A^{\tau_i}_{R\!N})^\Jmc$, by~(\ref{eq:R}),
    $\exists t\in {R\!N}^\Imc$ s.t. $d=\imath(t[{\tau_i}])$. Since
    $\chd{\tau_i}{\tau_j}$, then $\exists d'=\imath(t[{\tau_j}])$ and,
    by~(\ref{eq:Q}), $(d,d')\in Q_{\tau_j}^\Jmc$, while
    by~(\ref{eq:R}), $d'\in (A^{\tau_j}_{R\!N})^\Jmc$. Thus,
    $d\in (\exists Q_{\tau_j}\per A^{\tau_j}_{R\!N})^\Jmc$.
  \item
    $\Jmc\models~\exists^{\geq 2} Q_{\tau_j}\per\top\sqsubseteq \bot$.
    The fact that each $Q_{\tau_j}$ is interpreted as a funcional role is
    a direct consequence of the construction~(\ref{eq:Q}) and the fact
    that $\imath$ is an injective function.
  \end{itemize}
\item $\Jmc\models \gamma_{\textit{lobj}}(R\!N)$, for every
  ${R\!N\in\Rmc}$. Similar as above, considering the fact that each $\ell_{R\!N}$ is
  an injective function and equations~(\ref{eq:loc})-(\ref{eq:Aloc}).
\item $\Jmc\models {C_1^\dag\sqsubseteq C_2^\dag}$ and
  $\Jmc\models {R_1^\dag\sqsubseteq R_2^\dag}$. Since
  $\Imc\models {C_1\sqsubseteq C_2}$ and
  $\Imc\models R_1\sqsubseteq R_2$, It is enough to show the
  following:
  \begin{itemize}
  \item $d\in \Int{C} \text{ iff } d\in (C^\dag)^\Jmc$, for all \DLRpm\ concepts;
  \item $t\in \Int{R} \text{ iff } \imath(t)\in (R^\dag)^\Jmc$, for all \DLRpm\ relations.
  \end{itemize}
  Before we proceed with the proof, it is easy to show by structural
  induction that the following property holds:
  \begin{align}\label{prop:RN1}
    \text{If } \imath(t)\in R^{\dag\Jmc} \text{ then } \exists
    \imath(t')\in R\!N^{\dag\Jmc} 
    \text{ s.t. } t=t'[\tau(R)], \text{ for some } R\!N\in\R.
  \end{align}
  We now proceed with the proof by structural induction. The base
  cases, for atomic concepts and roles, are immediate form the
  definition of both ${C\!N}^\Jmc$ and ${R\!N}^\Jmc$.
  The cases where complex concepts and relations are constructed using
  either boolean operators or global reification are easy to show. We
  thus show only the following cases.\\
  Let $d\in(\lreif R\!N)^\Imc$. Then, $d=\ell_{R\!N}(t)$ with
  $t\in R\!N^\Imc$. By induction, $\imath(t)\in A_{R\!N}^\Jmc$ and, by
  $\gamma_{\textit{lobj}}({R\!N})$, there is a $d'\in\Delta^\Jmc$
  s.t. $(\imath(t),d')\in Q_{R\!N}^\Jmc$ and
  $d'\in (A_{R\!N}^l)^\Jmc$. By~(\ref{eq:loc}), $d'=\ell_{R\!N}(t)$
  and, since $\ell_{R\!N}$ is injective, $d'=d$. Thus,  $d\in(\lreif
  R\!N)^{\dag\Jmc}$.\\
  Let $d\in(\exists^{\geq q}[U_i] R)^\Imc$. Then, there are different
  $t_1,\ldots,t_q\in R^\Imc$ s.t. $t_l[U_i]=d$, for all
  $l=1,\ldots,q$. By induction, $\imath(t_l)\in R^{\dag \Jmc}$ while,
  by~(\ref{prop:RN1}), $\imath(t'_l)\in R\!N^{\dag\Jmc}$, for
  some atomic relation ${R\!N}\in\R$ and a tuple $t'_l$
  s.t. $t_l=t'_l[\tau(R)]$. By~$\gamma_{\textit{rel}}({R\!N})$
  and~(\ref{eq:Q}), $(\imath(t'_l),\imath(t_l))\in
  (\pth{\tau({R\!N})}{\tau(R)}^\dag)^\Jmc$ and  $(\imath(t_l),d)\in
  (\pth{\tau({R})}{\{U_i\}}^\dag)^\Jmc$. Since $\imath$ is injective,
  $\imath(t_l)\neq \imath(t_j)$ when $l\neq j$,
  thus, $d\in(\exists^{\geq q}[U_i] R)^{\dag\Jmc}$.\\
  Let $t\in(\selects{U_i}{C}{R})^\Imc$. Then, $t\in R^\Imc$ and
  $t[U_i]\in C^\Imc$ and, by induction, $\imath(t)\in R^{\dag \Jmc}$
  and $t[U_i]\in C^{\dag \Jmc}$. As before, by
  $\gamma_{\textit{rel}}(R\!N)$ and by~(\ref{eq:Q}) and (\ref{prop:RN1}),
  $(\imath(t),t[U_i])\in(\pth{\tau(R)}{\{U_i\}}^\dag)^\Jmc$.
   Since $\pth{\tau(R)}{U_i}^\dag$ is functional, then we have that
  $\imath(t)\in (\selects{U_i}{C}{R})^{\dag \Jmc}$.\\
  Let $t\in(\exists[U_1,\ldots,U_k] R)^{\Imc}$. Then, there is a tuple
  $t'\in R^\Imc$ s.t. $t'[U_1,\ldots,U_k]=t$ and, by induction,
  $\imath(t')\in R^{\dag\Jmc}$. As before, by
  $\gamma_{\textit{rel}}(R\!N)$ and by~(\ref{eq:Q}) and
  (\ref{prop:RN1}), we can show that
  $(\imath(t'),\imath(t))\in
  \pth{\tau(R)}{\{U_1,\ldots,U_k\}}^{\dag\Jmc}$ and thus
  $\imath(t)\in(\exists[U_1,\ldots,U_k] R)^{\dag\Jmc}$.\\
  All the other cases can be proved in a similar way. We now show the
  vice versa.

  \smallskip

  Let $d\in(\lreif R\!N)^{\dag\Jmc}$. Then, $d\in (A_{R\!N}^l)^\Jmc$
  and $d=l_{R\!N}(t)$, for some $t\in
  {R\!N}^\Imc$, i.e., $d\in(\lreif R\!N)^\Imc$.\\
  Let $d\in(\exists^{\geq q}[{U_i}] R)^{\dag \Jmc}$. Then, there are
  different  $d_1,\ldots,d_q\in\Delta^\Jmc$ s.t.
  $(d_l,d)\in (\pth{\tau(R)}{\{U_i\}}^\dag)^{\Jmc}$ and
  $d_l\in R^{\dag \Jmc}$, for $l=1,\ldots,q$. By induction, each
  $d_l=\imath(t_l)$ and $t_l\in R^\Imc$. Since $\imath$ is injective,
  then $t_l\neq t_j$ for all $l,j=1,\ldots,q$, $l\neq j$. We need to
  show that $t_l[U_i] = d$, for all $l=1,\ldots,q$. By~(\ref{eq:Q})
  and the fact that $(d_l,d)\in (\pth{\tau(R)}{\{U_i\}}^\dag)^{\Jmc}$,
  then $d=\imath(t_l[U_i])=t_l[U_i]$.\\
  Let $\imath(t)\in(\selects{U_i}{C}{R})^{\dag\Jmc}$.  Then,
  $\imath(t)\in R^{\dag \Jmc}$ and, by induction, $t\in R^\Imc$. Let
  $t[U_i]=d$. We need to show that $d\in C^\Imc$. By
  $\gamma_{\textit{rel}}(R\!N)$ and by~(\ref{eq:Q}) and
  (\ref{prop:RN1}),
  $(\imath(t),d)\in (\pth{\tau(R)}{\{U_i\}}^\dag)^\Jmc$, then
  $d\in C ^{\dag\Jmc}$ and, by induction, $d\in C ^{\Imc}$.\\
%
  Let $\imath(t)\in(\exists[U_1,\ldots,U_k] R)^{\dag\Jmc}$. Then, there is
  $d\in\Delta^\Jmc$ s.t.
  $$(d,\imath(t)) \in (\pth{\tau(R)}{\{U_1,\ldots,U_k\}}^\dag)^\Jmc$$ and
  $d\in R^\Jmc$.  By induction, $d=\imath(t')$ and
  $t'\in R^\Imc$. By~(\ref{eq:Q}), $\imath(t)=
  \imath(t'[U_1,\ldots,U_k])$, i.e., $t=t'[U_1,\ldots,U_k]$. Thus,
  $t\in(\exists[U_1,\ldots,U_k] R)^{\Imc}$.

\smallskip

($\Leftarrow$) Let $\Jmc=(\Delta^\Jmc,\cdot^\Jmc)$ be a model for the
knowledge base $\gamma(\KB)$.  Without loss of generality, we can
assume that $\Jmc$ is a \textit{tree model}. We then construct a model
$\Imc = (\Delta^\Imc, \cdot^\Imc, \rho, \imath, \ell_{R\!N_1},\ldots)$
for a \DLRpm knowledge base $\KB$. We set:
$\Delta^\Imc = \Delta^\Jmc$, $C\!N^\Imc = (C\!N^\dag)^\Jmc$, for every
atomic concept $C\!N\in\Cmc$, while, for every ${R\!N}\in\Rmc$, we set:
\begin{multline}\label{eq:RN}
  R\!N^\Imc = \{t=\langle U_1:d_1,\ldots,U_n:d_n\rangle \in T_{\Delta^\Imc}(\tau(R\!N))\mid
  \exists d\in A_{R\!N}^\Jmc \text{ s.t. }\\ (d,t[U_i])\in (\pth{\tau(R\!N)}{\{U_i\}}^\dag)^\Jmc \text{ for } i=1,\ldots,n\}.
\end{multline}
Since $\Jmc$ satisfies $\gamma_{\textit{rel}}(R\!N)$, then, for every
$d\in A_{R\!N}^\Jmc$ there is a unique tuple
$\langle U_1:d_1,\ldots,U_n:d_n\rangle \in R\!N^\Imc$, we say that $d$
\emph{generates} $\langle U_1:d_1,\ldots,U_n:d_n\rangle$ and, in
symbols, $d\to \langle U_1:d_1,\ldots,U_n:d_n\rangle$. Furthermore,
since \Jmc is tree shaped, to each tuple corresponds a unique $d$ that
generates it. Thus, let $d\to \langle U_1:d_1,\ldots,U_n:d_n\rangle$,
by setting $\imath(\langle U_1:d_1,\ldots,U_n:d_n\rangle) = d$ and
\begin{multline}\label{eq:iota}
  \imath(\langle U_1:d_1,\ldots,U_n:d_n\rangle[\tau_i])=d_{\tau_i}, \text{ s.t. }\\
  (d,d_{\tau_i})\in(\pth{\{U_1,\ldots,U_n\}}{\tau_i}^\dag)^\Jmc,
\end{multline}
for all ${\tau_i\in\mathscr{T}_{\{U_1,\ldots,U_n\}}}$, then, the function $\imath$ is as required.\\
By setting
\begin{multline}\label{eq:lobj}
  \ell_{R\!N}(\langle U_1:d_1,\ldots,U_n:d_n\rangle) =
  d, \text{ s. t. }\\
  (\imath(\langle U_1:d_1,\ldots,U_n:d_n\rangle),d)\in
  Q_{R\!N}^\Jmc,
\end{multline}
by $\gamma_{\textit{lobj}}(R\!N)$, both $Q_{R\!N}$ and its inverse are
interpreted as a functional roles by \Jmc, thus
the function $\ell_{R\!N}$ is as required.\\
It is easy to show by structural induction that the following property holds:
\begin{align}\label{prop:RN}
\text{If } t\in R^\Imc \text{ then } \exists t'\in R\!N^\Imc \text{
  s.t. } 
t=t'[\tau(R)], \text{ for some } R\!N\in\R.
\end{align}
We now show that $\Imc$ is indeed a model of $\KB$, i.e.,
$\Imc\models {C_1\sqsubseteq C_2}$ and
$\Imc\models R_1\sqsubseteq R_2$. As before, since
$\Jmc\models {C_1^\dag\sqsubseteq C_2^\dag}$ and
$\Jmc\models R_1^\dag\sqsubseteq R_2^\dag$, it is enough to show the
following:
  \begin{itemize}
  \item $d\in \Int{C} \text{ iff } d\in (C^\dag)^\Jmc$, for all \DLRpm\ concepts;
  \item $t\in \Int{R} \text{ iff } \imath(t)\in (R^\dag)^\Jmc$, for all \DLRpm\ relations.
  \end{itemize}
  The proof is by structural induction. The base cases are trivially
  true. Similarly for the boolean operators and global
  reification. We thus show only the following cases.\\
  Let $d\in(\lreif R\!N)^\Imc$. Then, $d=\ell_{R\!N}(t)$ with
  $t\in R\!N^\Imc$. By induction, $\imath(t)\in A_{R\!N}^\Jmc$ and, by
  $\gamma_{\textit{lobj}}({R\!N})$, there is a $d'\in\Delta^\Jmc$
  s.t. $(\imath(t),d')\in Q_{R\!N}^\Jmc$ and
  $d'\in (A_{R\!N}^l)^\Jmc$.  By~(\ref{eq:lobj}), $d=d'$ and
  thus, $d\in (\lreif R\!N)^{\dag\Jmc}$.
  \\
  Let $d\in(\exists^{\geq q}[{U_i}] R)^{\Imc}$. Then, there are
  different $t_1,\ldots,t_q\in R^\Imc$ s.t. $t_l[U_i]=d$, for all
  $l=1,\ldots,q$. For each $t_l$, by~(\ref{prop:RN}), there is a
  $t'_l\in R\!N^\Imc \text{ s.t. } t_l=t'_l[\tau(R)]$, for some
  $R\!N\in\R$, while, by induction, $\imath(t_l)\in R^{\dag \Jmc}$ and
  $\imath(t'_l)\in R\!N^{\dag\Jmc}$. Thus, $t'_l[U_i]=t_l[U_i]=d$ and,
  by~(\ref{eq:RN}),
  $(\imath(t'_l),d)\in (\pth{\tau(R\!N)}{\{U_i\}}^\dag)^\Jmc$ while,
  by~(\ref{eq:iota}),
  $(\imath(t'_l),\imath(t_l))\in
  (\pth{\tau(R\!N)}{\tau(R)})^{\dag\Jmc}$.
  Since \DLRpm allows only for knowledge bases with a projection
  signature graph being a multitree, then,
  $$\pth{\tau(R\!N)}{\{U_i\}}^\dag =
  \pth{\tau(R\!N)}{\tau(R)}^\dag\chain \pth{\tau(R)}{\{U_i\}}^\dag.$$
  Thus, $(\imath(t_l),d)\in (\pth{\tau(R)}{\{U_i\}}^\dag)^\Jmc$ and,
  since $\imath$ is injective, then, $\imath(t_l)\neq \imath(t_j)$
  when $l\neq j$. Thus,  $d\in(\exists^{\geq q}[{U_i}] R)^{\dag\Jmc}$.\\
  Let $t\in(\selects{U_i}{C}{R})^\Imc$. Then, $t\in R^\Imc$ and
  $t[U_i]=d\in C^\Imc$. By induction, $\imath(t)\in R^{\dag\Jmc}$ and
  $d\in C^{\dag\Jmc}$. As before, by~(\ref{eq:RN}),~(\ref{eq:iota}) and~(\ref{prop:RN}),
  we can show that
  $(\imath(t),d)\in (\pth{\tau(R)}{\{U_i\}}^\dag)^\Jmc$ and, since
  $\pth{\tau(R)}{\{U_i\}}^{\dag}$ is functional, then
  $\imath(t)\in(\selects{U_i}{C}{R})^{\dag\Jmc}$.
  \\
  Let $t\in(\exists[U_1,\ldots,U_k] R)^{\Imc}$. Then, there is a tuple
  $t'\in R^\Imc$ s.t. $t'[U_1,\ldots,U_k]=t$ and, by induction,
  $\imath(t')\in R^{\dag\Jmc}$. As before,
  by~(\ref{eq:iota}) and (\ref{prop:RN}), we can show that
  $(\imath(t'),\imath(t))\in
  \pth{\tau(R)}{\{U_1,\ldots,U_k\}}^{\dag\Jmc}$ and thus
  $\imath(t)\in(\exists[U_1,\ldots,U_k] R)^{\dag\Jmc}$.\\
  All the other cases can be proved in a similar way. We now show the
  vice versa.

  \smallskip

  Let $d\in(\lreif R\!N)^{\dag\Jmc}$. Then, $d\in (A_{R\!N}^l)^\Jmc$
  and, by $\gamma_{\textit{lobj}}({R\!N})$, there is a
  $d'\in\Delta^\Jmc$ s.t. $(d',d)\in Q_{R\!N}^\Jmc$ and
  $d'\in A_{R\!N}^\Jmc$. By induction, $d'=\imath(t')$ with
  $t'\in {R\!N}^\Imc$ and thus, $(\imath(t'),d)\in Q_{R\!N}^\Jmc$ and,
  by~(\ref{eq:lobj}), $\ell_{R\!N}(t') = d$, i.e., $d\in(\lreif R\!N)^\Imc$.
  \\
  Let $d\in(\exists^{\geq q}[{U_i}] R)^{\dag \Jmc}$.
  Thus, there are different $d_1,\ldots,d_q\in\Delta^\Jmc$ s.t.
  $(d_l,d)\in (\pth{\tau(R)}{\{U_i\}}^\dag)^{\Jmc}$ and
  $d_l\in R^{\dag \Jmc}$, for $l=1,\ldots,q$. By induction, each
  $d_l=\imath(t_l)$ and $t_l\in R^\Imc$. Since $\imath$ is injective,
  then $t_l\neq t_j$ for all $l,j=1,\ldots,q$, $l\neq j$. We need to
  show that $t_l[U_i] = d$, for all
  $l=1,\ldots,q$. By~(\ref{prop:RN}), there is a
  $t'_l\in R\!N^\Imc \text{ s.t. } t_l=t'_l[\tau(R)], \text{ for some
  } R\!N\in\R$ and, by~(\ref{eq:iota}),
  $(\imath(t'_l),\imath(t_l))\in
  (\pth{\tau(R\!N)}{\tau(R)}^\dag)^\Jmc$.
  Since $(\imath(t_l) ,d) \in (\pth{\tau(R)}{\{U_i\}}^\dag)^\Jmc$ and
  $\textsc{path}_{\mathscr{T}}$ is functional in \DLRpm,
  then, $(\imath(t'_l) ,d) \in (\pth{\tau(R\!N)}{\{U_i\}}^\dag)^\Jmc$
  and, by~(\ref{eq:RN}), $t'_l[U_i]=t_l[U_i] =d$.\\
  Let $\imath(t)\in(\selects{U_i}{C}{R})^{\dag\Jmc}$. Thus,
  $\imath(t)\in R^{\dag \Jmc}$ and, by induction, $t\in R^\Imc$. Let
  $t[U_i]=d$. We need to show that $d\in C^\Imc$. As before,
  by~(\ref{prop:RN}) and~(\ref{eq:iota}), we have that
  $(\imath(t),d)\in (\pth{\tau(R)}{\{U_i\}}^\dag)^\Jmc$. Then
  $d\in C ^{\dag\Jmc}$ and, by induction, $d\in C ^{\Imc}$.\\
  Let $\imath(t)\in(\exists[U_1,\ldots,U_k] R)^{\dag\Jmc}$. Then, there is
  $d\in\Delta^\Jmc$ s.t.
  $$(d,\imath(t)) \in (\pth{\tau(R)}{\{U_1,\ldots,U_k\}}^\dag)^\Jmc$$ and
  $d\in R^{\dag\Jmc}$.  By induction, $d=\imath(t')$ and
  $t'\in R^\Imc$. As before, by~(\ref{eq:iota}) and (\ref{prop:RN}), we
  can show that there is a tuple $t''\in R\!N$ s.t. $(\imath(t''),\imath(t))\in
  (\pth{\tau(R\!N)}{\{U_1,\ldots,U_k\}}^\dag)^\Jmc$ and thus,
  $t=t'[U_1,\ldots,U_k]$, i.e., $t\in(\exists[U_1,\ldots,U_k] R)^{\Imc}$.
 \hfill\qed
\end{enumerate}

\end{proof}

As a direct consequence of the above theorem and the fact that \DLR is a sublanguage of \DLRpm, we have that

\begin{corollary}
  Reasoning in \DLRpm is an \ExpTime-complete problem.
\end{corollary}


\section{Acknowledgements}
We thank Alessandro Mosca for working with us on all the preliminary work necessary to understand how to get these technical results.


\bibliographystyle{named}

\end{document}